\documentclass{article}
\usepackage{iclr2027_conference,times}

\iclrfinalcopy

\usepackage{graphicx}
\usepackage{adjustbox}
\usepackage{tikz}
\usetikzlibrary{arrows.meta, positioning, calc}
\usepackage{float}
\usepackage{enumitem}
\usepackage[utf8]{inputenc} 
\usepackage[T1]{fontenc}    
\usepackage{url}            
\usepackage{booktabs}       
\usepackage{amsfonts}       
\usepackage{nicefrac}       
\usepackage{microtype}      
\usepackage{xcolor}         
\usepackage{amsmath,amssymb,amsthm,mathtools}
\usepackage{thm-restate}
\usepackage{hyperref}       
\hypersetup{hidelinks}
 \title{Tool Use Reduces Depth-Induced Collapse \\ in OOD Reasoning}

 \newtheorem{theorem}{Theorem}[section]

\newtheorem{lemma}[theorem]{Lemma}
\newtheorem{corollary}[theorem]{Corollary}

\author{David Koplow$^1$ \hspace{1em} Tomer Galanti$^2$ \hspace{1em} Tomaso Poggio$^1$  \\
$^1$MIT, $^2$Texas A\&M University \\
}

\begin{document}
\maketitle

\begin{abstract}

Humans can apply ideas learned in one context to substantially different situations. We call this process of searching for and constructing novel recombinations of learned relationships to solve new problems \textit{out-of-distribution (OOD) reasoning}. The capacity for large language models (LLMs) to support OOD reasoning underpins proposals for generally intelligent systems. However, this property is challenging to measure because most problems admit many decompositions, some involving shallow subproblems and others involving subproblems that may have been memorized from the training data. This makes it difficult to determine how much compositional reasoning a model must actually perform. Uncertainty about training distributions, how to measure a datapoint's distance from a training distribution, and the exponential number of ways to decompose most tasks make it intractable to robustly measure any modern LLM's OOD-reasoning capacity on standard natural-language tasks. In this work, we introduce a benchmark in which a model progressively solves a Boolean circuit over $GF(2)$ from data. This benchmark is both minimal, isolating OOD reasoning from these confounding factors, and general, as any computable function can be represented as a sufficiently large $GF(2)$ polynomial. We find that standalone models’ next-step accuracy collapses as depth grows. In contrast, tool use through the synthesis and execution of code prevents this collapse in both small and frontier LLMs. These results indicate that synthesizing tools play a crucial role in supporting OOD reasoning.

\end{abstract}

\section{Introduction}

The ultimate goal of artificial intelligence is to build generally intelligent systems capable of solving the broad range of problems humans can solve. Yet humans and current machine-learning systems appear to differ in how they acquire this capacity. Humans can learn new concepts and relationships from a few examples, then apply them in substantially different contexts. Modern learning systems, despite being trained on internet-scale datasets, often fail to solve problems outside their training distribution.

A promising explanation for the impressive generalization capacity of humans, is that we are able to effectively leverage the compositional nature of most problems in the world \citep{Goyal2022InductiveBiases}. Compositional sparsity proposes that learnable problems can be decomposed into compositions of many sub problems, referred to here as transforms, that are each sparse, which can help circumvent the curse of dimensionality \citep{Danhofer2025CompositionalSparsity}. When we generalize, we are really just recomposing these learned sub-functions to find an answers to a novel problems \citep{Fu2024CompositionalGeneralization}. Reasoning can thus be understood as dynamically identifying, exploring, and evaluating ways to compose reusable transforms to solve a given problem \citep{Goyal2022InductiveBiases}. When both the problem and the decomposition required to solve it fall outside a model's training distribution, this constitutes out-of-distribution (OOD) reasoning. Reported solutions to previously unsolved problems, such as Open AI's claimed solution to the Navier–Stokes Millennium Prize Problem \citep{OpenAI2026NavierStokes}, and mechanistic findings indicate that chain-of-thought reasoning can facilitate compositional sub-function learning \citep{Li2023DissectingChainOfThought}, suggest that reasoning LLMs possess some capacity for OOD reasoning (Figure \ref{fig:motfig}). For this capacity to support artificial general intelligence, it must remain sufficiently reliable as problem complexity increases. Otherwise, declining reasoning accuracy would constrain how far a model can compositionally generalize beyond its training distribution. Many proposals for more general AI systems depend on this capacity not degrading catastrophically as problem complexity increases \citep{meyerson2025millionstep, li2024moreagents, shalevshwartz2025reasoning}. To date, however, no method has been proposed to measure performance degradation as a function of compositional complexity, let alone verify that any existing model avoids collapse when solving sufficiently deep problems.

\begin{figure}[t]
    \centering
    \includegraphics[width=1.0\linewidth]{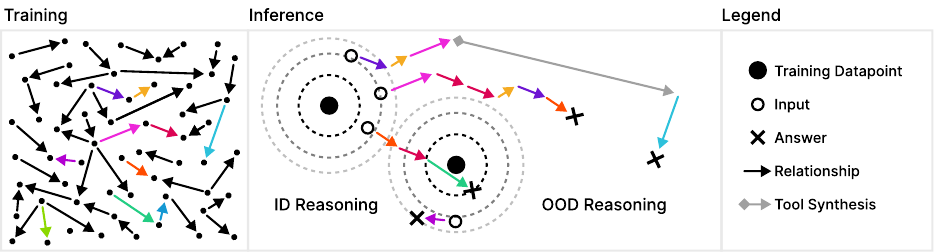}
\caption{The \textbf{left} panel provides an abstract depiction of training. The dots represent training datapoints in an idea space, and the arrows denote relationships linking these datapoints. The \textbf{center} panel zooms in on the region around a few datapoints and visualizes the problem-solving process during inference. The empty circles represent an initial query, and the LLM follows arrows corresponding to relationships learned in-distribution to reach an answer. If both the reasoning path and the answer lie near a training datapoint (dashed concentric circles), the process is considered in-distribution reasoning; otherwise, it is considered out-of-distribution reasoning. The long gray arrow illustrates the synthesis and application of a tool within the reasoning process, which expands the model's OOD reasoning capabilities.}
    \label{fig:motfig}
    \vspace{-0.5cm}
\end{figure}

This is because while compositional sparsity and generalization are incredibly attractive abstract theories, they are also fairly under-defined when applied to most natural language problems. There are a few sources of ambiguity when trying to measure a model's performance as a function of problem complexity: 

\begin{enumerate}[label=(\alph*), ref=(\alph*)]
    \item \label{item:abstraction-level} {\bf Abstraction:} Problems can be decomposed at different levels of abstraction, so compositional depth depends on the level chosen. Each function applied at a given level may itself involve internal reasoning, complicating comparisons of depth across otherwise similar solutions.
    \item \label{item:decomposition-choice} {\bf Decomposition:} A problem may admit multiple decompositions with substantially different difficulty. Some may draw on parts of the input that are easier to process using the learned transforms, or combine those transforms in simpler ways.
    \item \label{item:sampling-complexity} {\bf Sampling:} Decomposition and transform choices generate a combinatorial space of possible reasoning paths. Estimating success probabilities across this space in general requires exponentially many samples as depth increases.
\end{enumerate}

The combination of these factors means that measuring the compositional complexity of a general natural-language problem is ill-defined, depends on the model’s training distribution, and makes evaluating a specific model’s performance on that task intractable in non-shallow cases, because doing so may require exponentially many samples. 

To address these ambiguities, we introduce a benchmark based on stepwise Boolean-function reconstruction over $\mathrm{GF}(2)$. Every Boolean function admits a polynomial representation over this field. We consider a structured subclass of such functions and construct a dataset that provides a unique path for uncovering the terms of the corresponding $\mathrm{GF}(2)$ polynomial. Figure~\ref{fig:benchmark_workflow} visualizes this process, and Section~\ref{sec:benchmark} explains how the benchmark is constructed.

Because these polynomials are randomly generated and high-dimensional, they are exceedingly unlikely to appear in the model's training data. At the same time, the basic logical transformations required to recover successive terms must be within the model's capabilities for it to perform the task at all. By measuring compositional depth as the number of terms recovered along the unique reconstruction path, we control for factors~\ref{item:abstraction-level} and~\ref{item:decomposition-choice}. Each term depends only on the correct recovery of the preceding terms, irrespective of how those terms were recovered. This structure also allows us to circumvent the exponential sampling requirement in factor~\ref{item:sampling-complexity}: we can sample partial solutions at different depths and estimate the probability of recovering the next term without enumerating exponentially many complete reasoning paths.

\begin{figure}[t]
    \centering
    \begin{tikzpicture}[
        font=\footnotesize,
        box/.style={draw=black!55, rounded corners=2pt, align=center,
                    inner sep=5pt, minimum height=1.2cm},
        flow/.style={-Latex, semithick, draw=black!70}
    ]
        \node[box, text width=3.65cm] (input)
            {\textbf{Given $g$ and prefix $P_g$}\\Filter dataset for datapoints\\decryptable using $P_g$: $\mathsf S_g$};
        \node[box, fill=blue!5, text width=3.55cm, right=0.55cm of input] (llm)
            {\textbf{LLM}\\Decrypt labels using $P_g$;\\infer the missing term};
        \node[box, text width=3.05cm, right=0.95cm of llm] (score)
            {\textbf{Score reasoning step}\\$\hat t=t_{g+1}$: correct\\$\hat t\ne t_{g+1}$: incorrect};
        \draw[flow] (input) -- (llm);
        \draw[flow] (llm) -- node[below] {$\hat t$} (score);
        \coordinate (prediction) at ($(llm.east)!0.5!(score.west)$);
        \coordinate (nextstep) at ($(input.north)+(0,0.65cm)$);
        \draw[flow, dashed] (prediction) -- (prediction |- nextstep) --
            node[midway, above, font=\scriptsize]
            {During inference: proceed to the next term until all data is explained}
            (nextstep) -- (input.north);

        \node[box, fill=blue!5, minimum height=0.4cm,
              text width=3.0cm, below=0.55cm of llm] (tool) {Optional code execution};
        \draw[flow] ([xshift=-0.65cm]llm.south) --
            node[left, font=\scriptsize] {code} ([xshift=-0.65cm]tool.north);
        \draw[flow] ([xshift=0.65cm]tool.north) --
            node[right, font=\scriptsize] {output / error} ([xshift=0.65cm]llm.south);
    \end{tikzpicture}
    \caption{\small\textbf{Evaluating one reasoning step at a given depth.} Filter for datapoints decryptable using the supplied prefix $P_g$: $a_{g+1}=1$ and $a_j=0$ for $j>g+1$. This selection does not require knowing the missing support $S_{g+1}$. The LLM cancels the known contribution $\bigoplus_{j=1}^{g}a_jM_j(v)$ from the labels and predicts $\hat t$. Exact-match frequency estimates $\gamma_g$, compared with chance $1/\binom{p}{d-1}$. Repeated correct steps could reconstruct the circuit; the evaluation scores individual proposals from supplied prefixes.}
    \label{fig:benchmark_workflow}
\end{figure}

We examine the performance of small and frontier LLMs on this task and find that, as the problem’s compositional complexity $g$ increases, the probability $\gamma_g$ that the model correctly solves the next part of the problem degrades catastrophically. However, when the model is equipped with the ability to write and execute code, we do not observe the same collapse.

Although the rate of degradation and the depth at which collapse occurs will vary across tasks and depend on the factors discussed above, our benchmark demonstrates that increasing compositional depth alone can cause OOD reasoning to fail in standalone LLMs. We hope this work motivates future research into tool-augmented systems for general natural-language tasks that can support substantially greater out-of-distribution reasoning.

{\bf Contributions.\enspace} Our contributions are as follows:
\begin{enumerate}[leftmargin=*,itemsep=0pt, parsep=0pt, topsep=0pt]
    \item We introduce a Boolean circuit reconstruction benchmark over $\mathrm{GF}(2)$ that provides a reliable measure of problem compositional depth and enables controlled evaluation of stepwise OOD reasoning.

    \item We show that increasing compositional depth is associated with a collapse in next-step performance, even when LLMs can compose the constituent relationships. Smaller models collapse at lower depths than frontier models.

    \item We show that tool use substantially mitigates this collapse, even in small models, partly by externalizing computation to limit compounding errors. These findings establish a role for tools in supporting OOD reasoning beyond their conventional roles for augmenting knowledge and capabilities.
\end{enumerate}




\section{Related Work}\label{sec:related}

\textbf{Reasoning as search over partial solutions.\enspace}
A common view of LLM reasoning is that the model does not solve a problem in one step, but searches over intermediate states. This view appears in chain-of-thought prompting~\citep{wei2022cot}, self-consistency~\citep{wang2022selfconsistency}, Tree-of-Thought search~\citep{yao2023tot}, and iterative propose--critique--revise methods such as Reflexion~\citep{shinn2023reflexion}. These methods improve performance by expanding, filtering, or revising candidate reasoning traces. More broadly, tool-use and scratchpad methods give models an external medium for maintaining intermediate state and checking substeps~\citep{Nye2021,schick2023toolformer,hao2023toolkengpt,parisi2022talmtoolaugmentedlanguage,10.1145/3696410.3714825}. 

Our work studies a more basic question underlying these approaches: when a model is placed inside a search procedure, does it continue to assign nontrivial probability to the correct next step as the reasoning depth grows? Search-based inference is only useful if useful continuations remain reachable. 

\textbf{The Diligent Learner and the non-vanishing $\gamma$ assumption.\enspace}
The Diligent Learner framework~\citep{shalevshwartz2025reasoning} formalizes this process. It models reasoning as validator-guided search over semantic states, where progress depends on the probability that the model proposes a next action that keeps the current prefix completable. Denoting this probability by $\gamma$, the framework shows that search can remain efficient when $\gamma$ is bounded away from zero. 

Our paper targets the empirical gap left open by this theory. The Diligent Learner framework identifies the importance of non-vanishing stepwise progress, but does not measure whether modern LLMs satisfy this condition. In this work, we produce an empirical upper bound for this quantity and show it collapses without tool calls.  

\textbf{Benchmarks for reasoning and their shortcut problem.\enspace}
LLM reasoning is commonly evaluated on mathematical, logical, coding, and agentic benchmarks, including grade-school word problems~\citep{cobbe2021trainingverifierssolvemath}, competition mathematics~\citep{hendrycks2021measuring}, challenging subsets of BIG-Bench~\citep{suzgun2022challengingbigbenchtaskschainofthought}, software engineering tasks~\citep{jimenez2024swebench,yang2025swebench}, and interactive environments for tool-using agents~\citep{ALFWorld20,zhou2024webarenarealisticwebenvironment,qin2024toolllm,liu2025agentbenchevaluatingllmsagents}. Synthetic and controlled-distribution tasks also test compositional generalization and mathematical structure~\citep{frieder2025datamathematicalcopilotsbetter}. 

These benchmarks are valuable, but they usually measure end-to-end success rather than isolated stepwise progress. Many also admit multiple valid solution paths, contain redundant information, or permit partial-pattern strategies that do not require using all available evidence. This makes them poorly suited for testing whether a model can sustain OOD reasoning over a long chain of forced intermediate steps. 


\textbf{Tool use as externalized computation.\enspace}
Another line of work augments LLMs with programs, APIs, calculators, retrieval systems, or other tools~\citep{schick2023toolformer,hao2023toolkengpt,gao2023pal,chen2022program,cai2023large}. In many of these systems, tool use is framed as a way to expand what the model can do: the model delegates arithmetic, code execution, search, retrieval, or environment interaction to an external component. This perspective also appears in agent frameworks and multi-agent systems, where one model can call another model, verifier, planner, or environment-grounded actor as a specialized subroutine~\citep{wu2023autogen}.

Our results suggest a sharper role for tools. In our benchmark, the relevant information is already present, and the next step is recoverable in polynomial time. The failure of standalone LLMs is therefore not caused by missing knowledge, but by the need to carry out computation inside the model's own token-by-token reasoning process. This complements work such as program-aided reasoning and Reflexion: the relevant effect is not only that programs or feedback loops improve final answers, but that executable state can prevent the per-step success probability from decaying as the horizon grows. Tool synthesis changes this burden: the model can specify the computation once, execute it externally, inspect the result, and revise the implementation if needed. Thus, tools do not merely add capabilities; they can reduce the effective reasoning depth that must be handled internally by the LLM, helping prevent depth-induced collapse.

\section{Theory and Benchmark}\label{sec:theory}

In this section, we formalize the stepwise $\mathrm{GF}(2)$ reconstruction task used to empirically bound $\gamma_g$. By framing reasoning as an iterative decoding process, we require the model to extend an evolving partial solution (the prefix) by integrating it with a new batch of step-specific evidence. The core of our construction is a statistical obfuscation mechanism designed to block three prevalent shortcuts: (i) \emph{data-only} prediction, which attempts to infer the next term without access to the historical context; (ii) \emph{statistical-leakage} prediction, which exploits statistics of the data or problem specification to make informed guesses without fully incorporating the available evidence; and (iii) \emph{history-only} prediction, which relies solely on extrapolating from the prefix without considering the new observations. By preventing these shortcuts, we construct a class of problems with a unique decomposition at the chosen level of abstraction, whose compositional depth grows with the number of terms that must be recovered and cannot be reduced through a shallower solution path (factors~\ref{item:abstraction-level} and~\ref{item:decomposition-choice}). Stepwise evaluation then avoids the exponential sampling requirement associated with factor~\ref{item:sampling-complexity}.

\subsection{Problem formulation}\label{sec:theory:estimators}

We instantiate this framework with a Boolean-function reconstruction dataset over $\mathrm{GF}(2)$, using prefix-conditioned filtering at each depth.

{\bf Targets in algebraic normal form (ANF) over $\mathrm{GF}(2)$.\enspace}
Let $x=(a,v)\in\{0,1\}^{n+p}$, where $a=(a_1,\dots,a_n)\in\{0,1\}^n$ are \emph{address bits} and
$v=(v_1,\dots,v_p)\in\{0,1\}^p$ are \emph{payload bits}. We consider Boolean targets represented in algebraic normal form (ANF) as XORs of monomials,
$f(a,v) = \bigoplus_{j=1}^{n} t_j(a,v)$, where $t_j(a,v):= a_j\,M_j(v)$ and $M_j(v):=\prod_{i\in S_j} v_i$, with each support $S_j\subseteq[p]$ having fixed size $|S_j|=d-1$.
Thus, each term has payload degree $d{-}1$ and total degree $d$ in $(a,v)$.
An \emph{instance} is specified by the ordered sequence of supports $(S_1,\dots,S_n)$
(equivalently, the ordered ANF terms $(t_1,\dots,t_n)$), sampled once and then fixed.

{\bf Stepwise reconstruction game and step success.\enspace}
At step $g\in\{0,\dots,n-1\}$, the learner is given (i) the ordered prefix $P_g:=(t_1,\dots,t_g)$ and
(ii) a new labeled set $\mathsf S_g:=\{(x^{(k)},y^{(k)})\}_{k=1}^K$ filtered from the full dataset at step $g$.
The learner outputs a candidate monomial $\hat t$ and succeeds if and only if $\hat t=t_{g+1}$. We define the benchmark step-success probability as
\begin{equation}\label{eq:gamma_exact}
\gamma_g \;:=\; \Pr_{(P_g,\mathsf S_g,t_{g+1}),\,\hat t\sim\pi_\theta(\cdot\mid P_g,\mathsf S_g)}\big[\hat t=t_{g+1}\big],
\end{equation}
where the probability is taken over instance generation, sampling, and model stochasticity.
Because each instance specifies an ordered sequence, there is a unique correct extension $t_{g+1}$ at depth $g$,
so $\gamma_g$ directly operationalizes step success in our benchmark.

{\bf Estimator classes.\enspace}
We distinguish solvers according to their information access:
(i) \textbf{Diligent Estimator} ($\mathcal A_g$), with access to $(P_g,\mathsf S_g)$;
(ii) \textbf{Data-only Estimator} ($\mathcal B_g$), with access to $\mathsf S_g$ but not $P_g$;
(iii) \textbf{History-only Estimator} ($\mathcal C_g$), with access to $P_g$ but not $\mathsf S_g$; and
(iv) \textbf{Partial Estimator} ($\mathcal D_g$), with partial access to $P_g$ and $\mathsf S_g$.

Let $\gamma_g^{\mathcal X}$ denote the exact next-step success probability for class $\mathcal X$ at step $g$. Our benchmark is designed to enforce a separation of the form
\begin{small}
\begin{equation}\label{eq:separation_goal}
\min_g \gamma_g^{\mathcal A} \;\ge\; Q
\quad\text{while}\quad
\gamma_g^{\mathcal B},\,\gamma_g^{\mathcal C},\,\gamma_g^{\mathcal D}
\;\approx\; \tfrac{1}{\binom{p}{d-1}}.
\end{equation}
\end{small}
for a nontrivial constant $Q\in(0,1]$ bounded away from zero.

\begin{figure}[t]
    \centering
    \includegraphics[width=\linewidth]{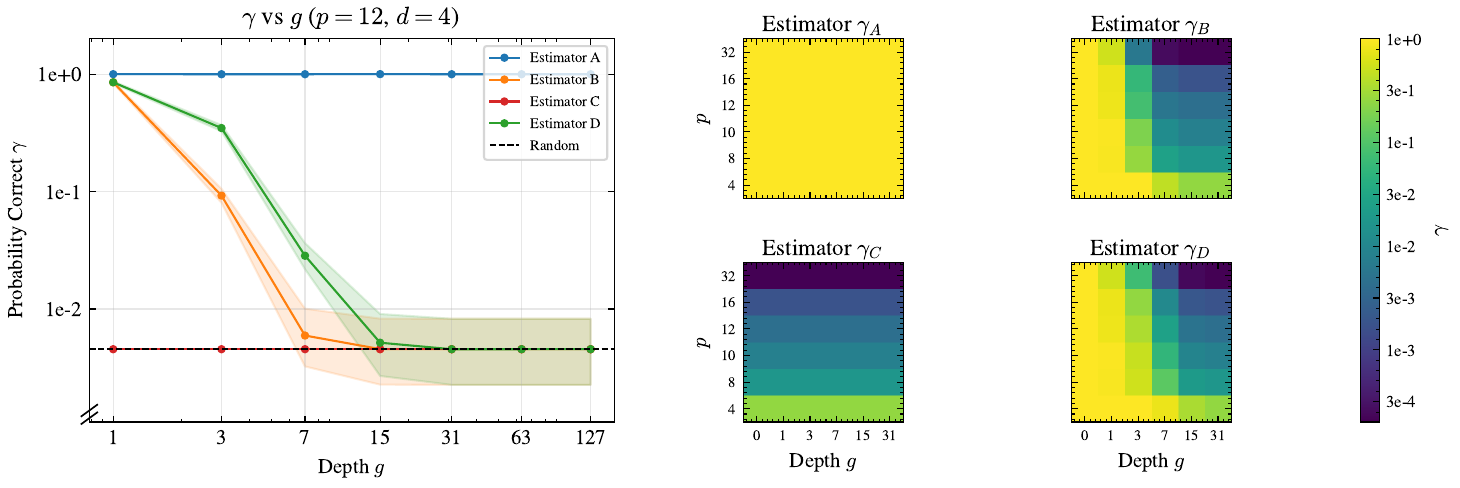}

\caption{\small
\textbf{Left:} Only estimators with access to both history and data sustain reliable next-step prediction. We plot step success $\gamma_g$ (the probability mass assigned to the correct next monomial) as a function of depth $g$ for each estimator class. Curves show the mean over $2000$ generated circuits at each depth, with shaded Jeffreys intervals.
\textbf{Right:} As both $g$ and $p$ increase, estimators with incomplete information tend to collapse toward zero probability mass on the correct next monomial. Only estimator $\mathcal{A}$ (history+data) consistently maintains high $\gamma_g$. In contrast, $\mathcal{B}$ (data-only) and $\mathcal{D}$ (partial) often collapse toward zero, while $\mathcal{C}$ (history-only) remains near chance. The heatmap is based on $200$ generated circuits for each hyperparameter setting.}
    
    \label{fig:gamma_simulations}
\end{figure}

\subsection{Statistical Obfuscation: Construction and Guarantees}\label{sec:theory:construction_and_guarantees}

{\bf Theoretical objectives.\enspace}
Our benchmark is designed to block shortcut reasoning so that problem complexity scales predictably with depth. To achieve this, we introduce a \textit{statistical obfuscation sampler} in which the revealed history provides the information needed to unmask the correct next term. The construction targets three sources of shortcut behavior:
\begin{enumerate}[leftmargin=*]
    \item \emph{No history-only shortcuts:} The sequence of prior steps provides no predictive information about the \emph{next} step without the new data.
    \item \emph{No statistical leakage:} Monomial evaluations are nearly balanced, preventing trivial prediction from marginal label statistics.
    \item \emph{No data-only shortcuts:} Without the revealed history, individual observations become increasingly masked as depth grows.
\end{enumerate}

Below, we describe the instance distribution and sampler and establish these properties. Because observations within a batch share the same latent prefix supports, our theoretical masking guarantee is row-wise; we separately test empirically whether joint correlations across a finite batch provide a practical data-only shortcut.

{\bf Instance generation.\enspace}
We first generate the underlying \emph{instance}. We sample an ordered sequence of payload supports $S_1,\dots,S_n\subseteq[p]$, with $|S_j|=d-1$, independently and uniformly with replacement. This fixes the ordered Algebraic Normal Form (ANF) terms $(t_1,\dots,t_n)$.

Because the supports are sampled independently, the revealed history contains no information about the upcoming support. This rules out history-only prediction beyond prior guessing.

\begin{restatable}[History-only is prior guessing]{lemma}{HistoryOnly}\label{lem:history_only}
Assume the instance distribution samples supports $S_1,\dots,S_n$ i.i.d.\ uniformly with replacement from $\{S\subseteq[p]:|S|=d-1\}$. Then, for any $g<n$, conditioned on the revealed prefix $P_g$, the next support $S_{g+1}$ remains uniform over $\{S\subseteq[p]:|S|=d-1\}$ and independent of $P_g$. Consequently, any history-only estimator satisfies $\Pr[\widehat S=S_{g+1}]\le \tfrac{1}{\binom{p}{d-1}}$.
\end{restatable}

{\bf Payload distribution.\enspace}
To prevent trivial statistical leakage, we fix a Hamming weight $w$ and sample payloads uniformly from $\{v\in\{0,1\}^p:\|v\|_0=w\}$. Under this distribution, every monomial fires with probability
$\rho(w)=\binom{w}{d-1}/\binom{p}{d-1}$
(Lem.~\ref{lem:monomial_weight}, App.~\ref{appendix:lemma-proofs}). We choose a weight $w^\star<p$ that makes $\rho(w^\star)$ as close to $1/2$ as possible, making monomial evaluations nearly balanced and suppressing marginal-bias shortcuts.

{\bf Obfuscation sampling.\enspace}
We next define the \emph{step-$g$ sampler}. Given an instance and depth $g\in\{0,\dots,n-1\}$, the sampler returns a labeled set
$\mathsf S_g=\{(a^{(k)},v^{(k)},y^{(k)})\}_{k=1}^K$
of i.i.d.\ examples generated as follows:
(i) set the target address bit $a_{g+1}=1$ and all future bits $a_j=0$ for $j>g+1$;
(ii) sample the prefix address bits $a_1,\dots,a_g\stackrel{i.i.d.}{\sim}\mathrm{Bernoulli}(1/2)$;
(iii) sample $v$ uniformly from $\{v:\|v\|_0=w^\star\}$; and
(iv) output $y:=f(a,v)$.

Because $a_{g+1}=1$ and all subsequent address bits are zero, the label decomposes into a prefix-dependent mask and the next-term signal:
\begin{equation}\label{eq:otp_decomp}
y
=
\underbrace{\Big(\bigoplus_{j=1}^{g}a_jM_j(v)\Big)}_{\text{prefix obfuscation}}
\oplus
\underbrace{M_{g+1}(v)}_{\text{next-term signal}}.
\end{equation}

A diligent solver can compute the mask from $(P_g,x)$ and therefore recover the signal. A data-only estimator does not know the prefix supports. After marginalizing over them, Lem.~\ref{lem:bayes_masking} shows that for a row with $m$ active prefix bits, the bias of the masking channel away from a fair XOR is
$\tfrac12|1-2\rho|^m$.
Since $a_1,\dots,a_g\stackrel{i.i.d.}{\sim}\mathrm{Bernoulli}(1/2)$, a typical row has $m\approx g/2$, so this bias decays exponentially with depth whenever $\rho$ is bounded away from $0$ and $1$. Cor.~\ref{cor:batch_masking} further shows that, for fixed $K$, with probability approaching one exponentially quickly in $g$, every row in the batch contains $\Omega(g)$ active prefix terms and is therefore strongly masked.

This is a row-wise guarantee: because all rows share the same latent prefix supports $(S_1,\dots,S_g)$, it does not by itself exclude information carried by correlations across rows. The estimator simulations in Figure~\ref{fig:gamma_simulations} test whether such finite-depth joint correlations yield a practical data-only shortcut in the regimes used by our experiments.

{\bf Recoverability in polynomial time for diligent solvers.\enspace}
Given the revealed prefix, a diligent solver can remove the mask in Eq.~\eqref{eq:otp_decomp} and reduce each step to recovering the support of a single Boolean conjunction. App.~\ref{appendix:poly-time-recovery} gives an explicit polynomial-time decoder.

{\bf Sample complexity.\enspace}
Once the prefix mask is removed, the number of step-specific samples required to recover the next support does not grow with depth $g$. Thm.~\ref{thm:recoverability_poly} and Cor.~\ref{cor:recoverability_hp} give a sufficient sample complexity of
$O(\log(p/\delta))$
in the benchmark regime of fixed degree, constant firing probability, and $\alpha$ bounded away from $1$. The computation required to remove the known prefix remains polynomial in $g$ and $p$. Thus, increasing depth does not make the full-information statistical recovery problem harder in sample complexity, while the masking bias of individual data-only observations decreases exponentially with depth.

\subsection{Benchmark Protocol}
\label{sec:benchmark}

This section describes how the construction in Sec.~\ref{sec:theory} is instantiated as an LLM benchmark.

{\bf Prompt structure.\enspace} Each step of the reconstruction game is presented as a single-turn completion problem. The variables $(a,v)$ are flattened into a single ordered sequence
$x_0,\dots,x_{N-1}$. Each prompt contains three components:
(1) global metadata, including the ambient dimension $N$ and monomial degree;
(2) the revealed prefix $P_g$, written as logical clauses; and
(3) the step-specific evidence $\mathsf S_g$, formatted as $K=32$ space-separated binary rows to reduce tokenization artifacts.

The prompt also specifies the active address variable $x_g$ and the partition point $n$. Thus, the model only needs to identify the remaining $d-1$ payload variables of $M_{g+1}$ among indices in $[n,N-1]$. Since models do not always follow the requested output format, we use a regex-based parser to extract the predicted set of payload indices from the completion.

{\bf Validation and evaluation.\enspace} By construction, each benchmark step has a unique correct continuation; see Section~\ref{sec:theory:estimators}. Validation therefore reduces to checking whether the predicted payload-index set exactly matches the ground-truth set. This gives an $O(d)$ validator, avoiding the high cost of verifying arbitrary intermediate circuit descriptions.

We estimate the depth-dependent step-success probability $\gamma_g$ using exact-match accuracy. To measure how proposal quality changes with reasoning depth, we evaluate at exponentially spaced depths, $g = 2^k - 1$. This isolates depth-induced degradation in next-step prediction while keeping the evaluation focused on the same local quantity required by the Diligent Learner analysis.

\section{Empirical Results}\label{sec:results}

We now validate our benchmark through experiments with perfect Bayesian estimators and compare the performance with that of modern LLMs. 

\begin{figure}[t]
    \centering
    \includegraphics[width=\linewidth]{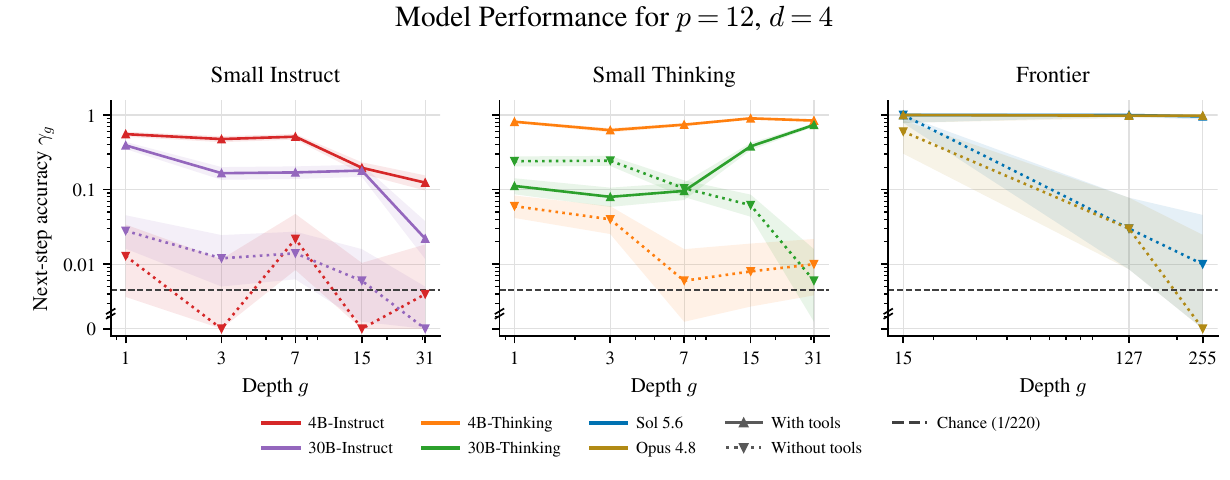}
    \caption{\small
\textbf{Left:} Tool use improves next-step prediction for small instruct models, although accuracy still declines with depth. We plot step success $\gamma_g$ versus circuit depth $g$ for Qwen3-2507 4B- and 30B-Instruct under adversarial sampling ($p=12$, $d=4$).
\textbf{Center:} With tools, 4B-Thinking sustains high accuracy and 30B-Thinking recovers near-perfect performance at larger depths; without tools, both models approach chance.
\textbf{Right:} Tool use sustains frontier-model accuracy at large depths. At $g=255$, Opus 4.8 and Sol 5.6 achieve near perfect accuracy with tools, versus negligible without tools. Solid and dotted curves denote performance with and without tools, respectively; shading indicates $95\%$ Jeffreys intervals, and the dashed line marks chance ($1/220$). Together, these results show that tool use substantially mitigates depth-induced collapse on this benchmark, and doesn't experience any collapse on the range tested for thinking models.}


    \label{fig:small_frontier_llm_gamma_vs_g}

\end{figure}

\subsection{Estimator Simulations}

We first check that the benchmark operates as designed. Sec.~\ref{sec:theory:estimators} introduced four classes of estimators, each with different levels of access to the prefix and to the new evidence. The theoretical analysis predicts that a solver that can use both sources of information should succeed, whereas solvers that rely solely on the history, solely on the data, or on only partial access should fail. Figure~\ref{fig:gamma_simulations} exhibits the expected behavior. For $g \ge 31$, no partial-information estimator is able to solve the task and as $g$ increases the performance of all estimators besides the full information estimator $\mathcal{A}$ degrades regardless of the number of inputs $p$. Additional non-adversarial comparisons are shown in Figure~\ref{fig:adv-comp}. This demonstrates that the benchmark is indeed behaving as it was designed.

\subsection{Small LLMs}~\label{sec:results:small}

We next ask whether small language models behave more like the diligent estimator or like the partial-information baselines. To test this, we evaluated four Qwen3-2507 models: \texttt{4B-Instruct}, \texttt{4B-Thinking}, \texttt{30B-A3B-Instruct}, and \texttt{30B-A3B-Thinking} \citep{qwen3_2507}. We ran inference in vLLM on $3000$ generated instances, split evenly across $g\in\{0,1,3,7,15,31\}$, using adversarial sampling with $(p,d)=(12,4)$. Consistent decoding required long contexts ($32{,}768$ tokens for the 4B models and $81{,}920$ for the 30B models). All small LLM results were collected in under $80$ hours of A100 compute.

Figure~\ref{fig:small_frontier_llm_gamma_vs_g} shows a clear depth-induced collapse. All models degrade as $g$ increases, despite the fact that the next step remains recoverable in polynomial time from the information provided (Thm.~\ref{thm:recoverability_poly}). \texttt{4B-Instruct} performs indistinguishably from random guessing even at the easiest settings, so we omit it for clarity. The stronger models perform better at shallow depths, but their next-step accuracy still drops sharply as the reasoning horizon grows. In particular, \texttt{Qwen3-30B-A3B-Thinking} improves substantially over its instruct counterpart at small $g$, yet still deteriorates and collapses by $g=31$, approaching the trivial baseline $\gamma_{\mathrm{triv}}$.

This pattern matches our predictions. The benchmark is designed so that progress requires using more and more of the revealed prefix to cancel the mask and isolate the next signal. A model that cannot reliably carry that prefix-conditioned computation forward should therefore drift toward the partial-information regime as depth increases.

To quantify how much of the prefix each model effectively integrates, we compute $k^\star(g)$, the prefix length at which estimator $\mathcal{D}$'s predicted accuracy matches the LLM's observed accuracy at depth $g$ (Figure~\ref{fig:k_star_trajectories}). Without tools, $k^\star$ remains at zero because the error roate is high enough that it's not destingishible from a perfect estimator with a that forgets one of the few prefixes available. It then grows with slope $f<1$ for the 30B variants ($f=0.71$--$0.85$), so the gap between $k^\star$ and $g$ widens with depth and accuracy diverges from the full-integration curve. 4B-Thinking's slope of $1.04$ is misleading: at small $g$, partial integration can yield correct answers without full prefix integration, so $k^\star$ tracks $g$ until $g\!\geq\!15$ when full integration becomes necessary and accuracy collapses to the random-guess boundary. 4B-Instruct without tools is omitted as its accuracy was at chance throughout. With tools, the warm-up disappears, slopes reach $0.96$--$1.08$ across all four models, and error rate is essentially independent of depth.

We then tested whether the four Qwen3-2507 models could overcome depth-induced collapse when given the ability to run code. We parsed code blocks from model outputs and executed them on-device in a secure \texttt{RestrictedPython} sandbox, appending the result or error message to the context and allowing up to $10$ iterations \citep{restrictedpython}.



\subsection{Frontier LLMs}\label{sec:results:frontier}
We next evaluate whether depth-induced collapse persists in frontier systems. We test Sol 5.6 and Opus 4.8 with high reasoning effort at $g\in\{15,127,255\}$, using $(p,d)=(12,4)$ and $K=32$ labeled rows per instance. Both models and conditions use the same circuit pools; the tool condition exposes the provider's code-execution tool, while the no-tool condition exposes no tools. Per model, we allocate $10$ requests per condition at $g=15$, and $100$ without tools and $50$ with tools at each larger depth, with tool instances drawn from the no-tool pool. This allocation was made in order to reduce inference costs guided by the trends in the smaller LLM performance. Figure~\ref{fig:small_frontier_llm_gamma_vs_g} reports exact-match accuracy among completed responses with $95\%$ Jeffreys intervals. Sampling is concentrated at the larger depths, so these measurements do not locate the precise onset of collapse. Full generation, token-budget, and scoring details are in App.~\ref{appendix:new-experiments}.

{\bf Frontier models also exhibit depth-induced collapse.\enspace} Without tools, Sol and Opus achieve $10/10$ and $6/10$ correct answers at $g=15$, respectively. At $g=127$, both fall to $3/100$; at $g=255$, accuracy falls to $1/100$ for Sol and $0/100$ for Opus, near the trivial baseline $\gamma_{\mathrm{triv}}=1/220$. Thus, substantial shallow-depth success does not translate into reliable next-step prediction at larger depths. All high-depth no-tool requests completed, so failed API requests do not explain this collapse.

{\bf Tool use stabilizes $\gamma_g$ over long horizons.\enspace} With tools, both models answer all $10$ shallow-depth instances correctly. At $g=127$, Sol achieves $50/50$ and Opus $49/50$; at $g=255$, they achieve $44/46$ and $49/50$ among completed responses. Four Sol tool requests at $g=255$ ended in server-overload errors after retries. Counting these as failures gives $44/50=88\%$, still far above the $1\%$ no-tool accuracy. The contrast is consistent with the small-model results: the ability to generate, execute, and revise code substantially mitigates depth-induced collapse.

Each benchmark step requires two operations: inferring the constraints from the prefix and new examples, then executing the corresponding GF(2) computation. Standalone LLMs perform both token-by-token, so as depth grows the prefix-conditioned computation lengthens. Tool use separates these roles by letting the model specify the computation in code, delegate execution to an interpreter, and revise based on observed results. The results are consistent with this externalization of deterministic computation reducing the internal reasoning burden and supporting stepwise OOD reasoning.

Overall, the frontier-model results reinforce the main conclusion of the benchmark. Frontier models remain susceptible to depth-induced collapse, while tool use changes the computational structure of the problem. By moving exact symbolic execution out of the model and into an executable artifact, tool-enabled systems maintain a much more stable step-success probability over the evaluated reasoning depths.

\section{Discussion}

In this work, we construct a benchmark based on Boolean circuit reconstruction that supports evaluating how well a model can reason OOD. This capacity is central to theories of learnability and generalization, though in practice it is ill-defined and intractable to measure reliably on standard natural-language tasks. We find that on this benchmark, next-step accuracy in small and frontier standalone LLMs collapses as depth grows, while code synthesis and execution substantially mitigate this collapse. This benefit arises in part from externalizing computation, reducing the compounding of errors within each reasoning step. These findings establish that depth-induced collapse can occur and that tool use can mitigate it. The depth at which this collapse occurs and the tools that could mitigate it in practice depend on the task. We hope future work builds on these findings to develop tools that support OOD reasoning on general natural-language problems.

\newpage






\newpage

\bibliography{references}
\bibliographystyle{iclr2027_conference}

\onecolumn

\appendix

\section{Other Experiments}\label{appendix:new-experiments}



 {\bf Execution-control Ablation.\enspace} Do the benefits of tool use in OOD reasoning arise from the opportunity to self-reflect and revise, a shift in the search space induced by a strong coding prior that helps identify a solution, or the actual execution of code? We run two experiments to examine these possibilities.

First, after a successful tool-assisted step at $g=31$, we retain the full interaction history, including the program and execution output, and inform the model that further execution is unavailable. We provide new examples for the next step at $g=32$ and allow an initial prediction followed by up to three revisions. Despite access to a successful program and opportunities to self-reflect and revise, both models perform near chance, with no improvement across revisions: 4B-thinking accuracy decreases from $0.2\%$ to $0\%$, while 30B-thinking accuracy remains at $0.4\%$ (Figure~\ref{fig:execution_control}). Thus, these capacities do not explain the tool-enabled performance gain in this benchmark.

Second, we vary the number of tool calls to examine the benefits of execution and subsequent revision. At $g=127$, a single tool call is often sufficient for the 30B-thinking model, whereas the 4B-thinking model requires additional execution-and-revision cycles to solve the problem reliably (Figure~\ref{fig:tool_calls_4b_thinking}).

Together, these experiments support the importance of code execution: self-reflection, revision, and access to program prior provide little benefit without it, while revising code using execution feedback can be particularly helpful for less capable models.





\begin{figure}[htbp]
    \centering
    \includegraphics[width=\linewidth]{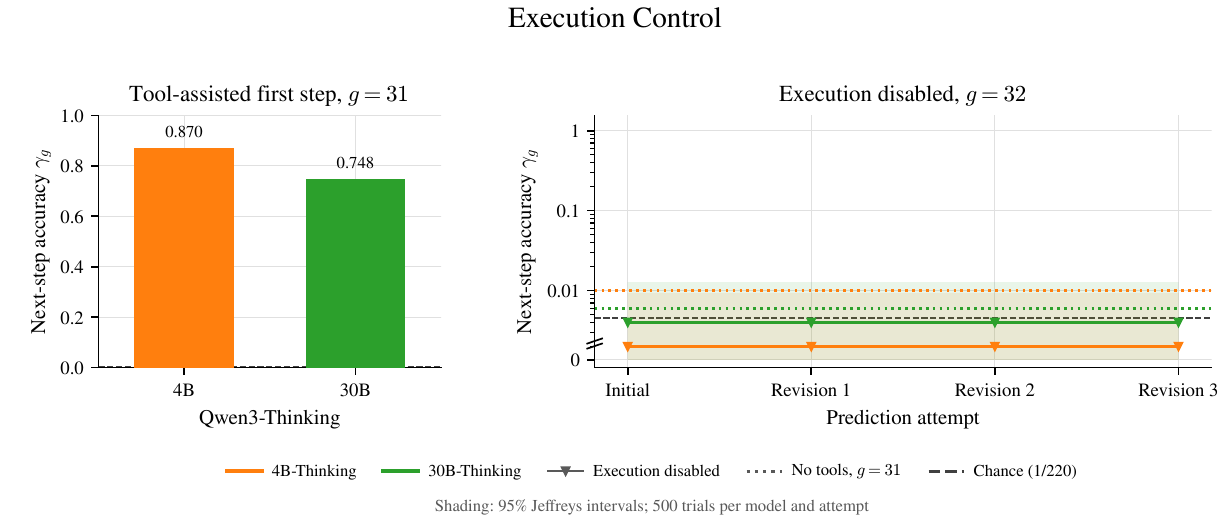}
    \caption{\small\textbf{Additional turns do not restore accuracy after execution is disabled.} Left: original tool-enabled accuracy at $g=31$. Right: next-step accuracy at $g=32$ with retained history but no further execution; dashed lines show the original no-tool reference at $g=32$. The panels use different vertical scales to make the $0$--$0.4\%$ control results visible. Values are the reported aggregate accuracies; colored regions represent Jeffery's intervals.}
    \label{fig:execution_control}
\end{figure}

\begin{figure}[t]
    \centering
    \begin{minipage}[c]{0.45\linewidth}
        \centering
        \includegraphics[width=\linewidth]{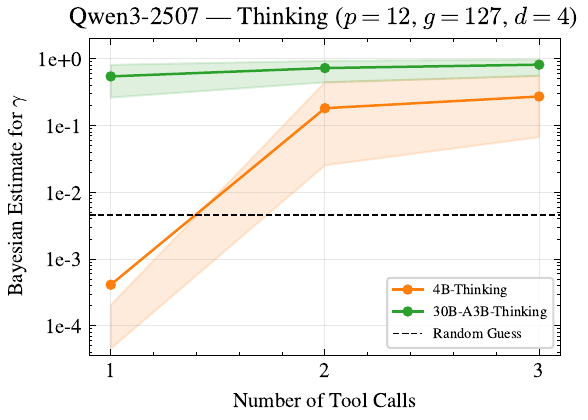}
    \end{minipage}
    \hfill
    \begin{minipage}[c]{0.5\linewidth}
        \caption{\small
        \textbf{Iterative tool use is essential for small models.}
        For \texttt{Qwen3-4B-Thinking}, next-step accuracy improves as the model
        is allowed more code execution and revision cycles. A single generated
        program is often insufficient at large depth, while repeated execution
        feedback enables the model to repair errors and recover the
        prefix-conditioned computation.
        }
        \label{fig:tool_calls_4b_thinking}
    \end{minipage}
\end{figure}

{\bf Token use and latency.\enspace} Aggregate token accounting at $g=31$ reports a largest observed increase in average total token use of approximately $35\%$, for 4B models requiring multiple calls. Tool-enabled 30B runs frequently use fewer total tokens, consistent with code execution replacing in-context symbolic calculation. Sandbox execution takes less than one second, so model inference dominates latency. These summaries indicate a model-dependent computational cost rather than a uniform token overhead. This is visualized in Figure \ref{fig:avg_tokens}.

\begin{figure}[htbp]
    \centering
    \includegraphics[width=\linewidth]{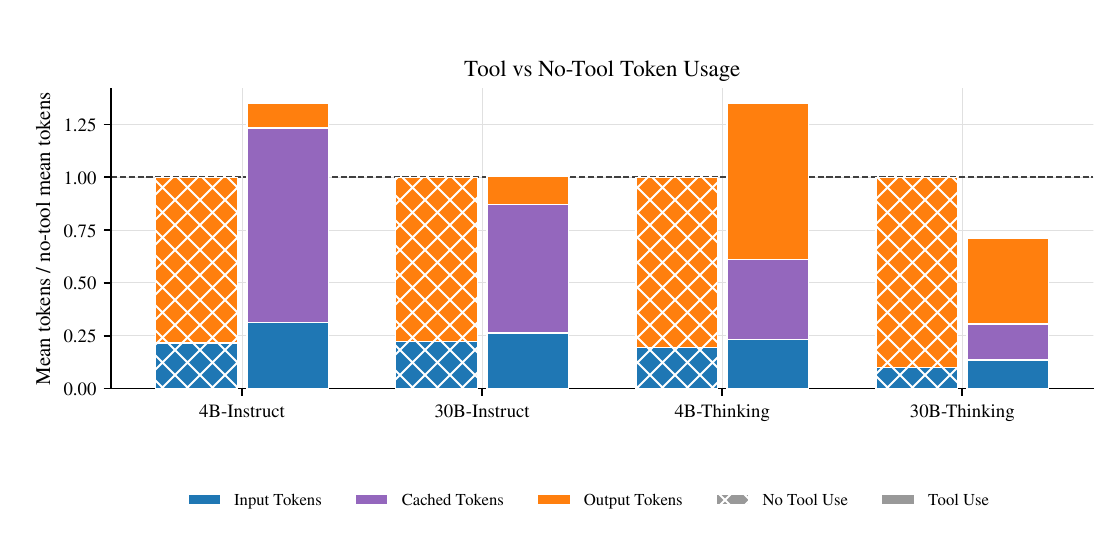}
\caption{\small\textbf{Token usage is similar with and without tools.} The plot shows the number of input, output, and cached tokens used to predict the next term at $g=31$. Token counts are averaged over 500 runs and divided by the corresponding average without tools for the same model and depth. Although each tool invocation includes all previously generated tokens as input, the shorter reasoning sequences keep overall token usage similar across conditions, with a maximum deviation of approximately $35\%$. If accessing through API, tool use is likely still cheaper across the board even when it consumes more tokens because cached tokens are substantially discounted and code execution incurs minimal cost and latency. Reusing cached tokens can also reduce inference time, allowing tool-enabled models to solve these problems faster.}
\label{fig:avg_tokens}
\end{figure}

\section{Other Simulations}~\label{appendix:other-simulations}

See Figs.~\ref{fig:adv-comp}, \ref{fig:min_gamma} and \ref{fig:tool_calls_4b_thinking}.

\begin{figure}
    \centering
    \includegraphics[width=0.5\linewidth]{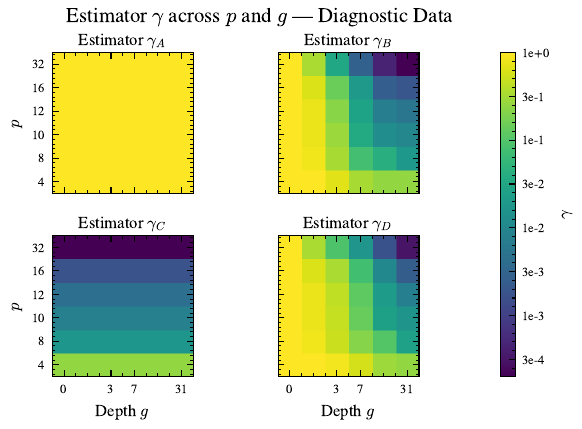}
    \caption{\small As both $g$ and $p$ increase, the probability of an estimator with imperfect information begins to collapse to zero. Only Estimator $\mathcal{A}$ is able to consistently produce the next monomial. The above heatmap was constructed through generating $200$ circuits for each combination of hyperparameters and computing the corresponding $\gamma_g$ for each $p$. This diagram shows the performance when not using the adversarial dataset construction. }
    \label{fig:adv-comp}
\end{figure}

\begin{figure}
    \centering
    \includegraphics[width=0.5\linewidth]{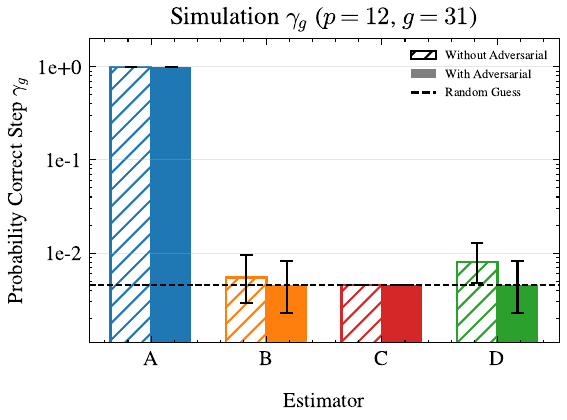}
    \caption{ The figure above shows how $\gamma_g$ changes for Estimators $\mathcal{A}$, $\mathcal{B}$, $\mathcal{C}$, and $\mathcal{D}$ when the data is biased adversarially (as in the results so far) preventing identification of monomials from frequency statistics, and without de-biasing.}
    \label{fig:min_gamma}

\end{figure}

\begin{figure}[t]
\centering
\includegraphics[width=\linewidth]{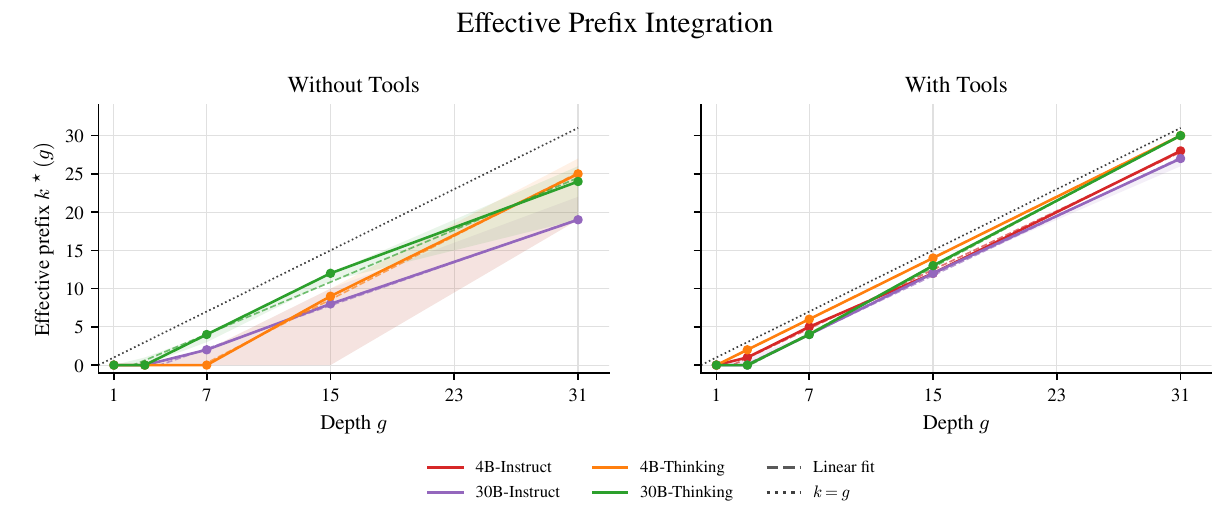}
\caption{\textbf{Slopes below 1 without tools imply diverging error with depth; slopes near 1 with tools imply error rate independent of problem depth.}
Estimator $\mathcal{D}_k$ is the partial-information solver from Sec.~\ref{sec:theory:estimators} given access to only the first $k$ revealed prefix monomials, with predicted next-step accuracy $\gamma_\mathcal{D}(p, g, k)$. For each $(\text{model}, g)$ pair, $k^\star(g) = \arg\min_k |\gamma_{\text{obs}}(g) - \gamma_\mathcal{D}(p, g, k)|$ is the prefix length under which $\mathcal{D}_k$ would reproduce the LLM's observed accuracy, i.e.\ the effective prefix the LLM behaves as if it integrated. Shaded bands are 95\% Jeffreys intervals on $\gamma_{\text{obs}}$ propagated to $k$. Dashed lines are linear fits $k^\star \approx \max(0, fg+a)$; the dotted line $k=g$ marks full integration; the solid black line marks the random-guess boundary, below which $\mathcal{D}_k$ falls to chance and $k^\star$ is unidentified. 4B-Instruct without tools is omitted: its accuracy was indistinguishable from chance at every depth.
Without tools, slopes are below $1$ ($f=0.71$--$1.04$) so the gap between $k^\star$ and $g$ grows with depth, and accuracy correspondingly diverges from full-integration performance. 4B-Thinking's slope of $1.04$ is misleading: at small $g$, partial integration can yield correct answers by chance, so $k^\star$ tracks $g$ until $g\!\geq\!15$ when full integration becomes necessary and accuracy collapses, widening the band to the random-guess boundary. With tools, slopes are essentially $1$ ($0.96$--$1.08$) and bands stay tight, so the error rate is independent of problem depth.}

\label{fig:k_star_trajectories}
\end{figure}

\section{Deferred Lemma Statements and Proofs}~\label{appendix:lemma-proofs}

Below we state and prove the full versions of the lemmas summarized in Section \ref{sec:theory:construction_and_guarantees}.

\begin{lemma}[Monomial firing probability at fixed Hamming weight]\label{lem:monomial_weight}
Fix integers $p\ge d-1\ge 1$ and $w\in\{0,\dots,p\}$.
Let $v$ be uniform over the Hamming sphere $\{v\in\{0,1\}^p:\|v\|_0=w\}$, and fix any $S\subseteq[p]$ with $|S|=d-1$.
Define $M_S(v):=\prod_{i\in S} v_i$.
Then
\[
\Pr \big[M_S(v)=1\big]
=
\begin{cases}
\displaystyle \frac{\binom{w}{d-1}}{\binom{p}{d-1}}
= \frac{\binom{p-(d-1)}{w-(d-1)}}{\binom{p}{w}}
& \text{if } w\ge d-1,\\[8pt]
0 & \text{if } w<d-1.
\end{cases}
\]
\end{lemma}

\begin{proof}
The event $M_S(v)=1$ is exactly the event that all coordinates in $S$ are contained in the support of $v$.
If $w<d-1$, this is impossible.
Otherwise, after requiring the $d-1$ coordinates in $S$ to be active, the remaining $w-(d-1)$ active coordinates may be chosen from the other $p-(d-1)$ coordinates.
Since $v$ is uniform over all $\binom{p}{w}$ supports of size $w$,
\[
\Pr[M_S(v)=1]
=
\frac{\binom{p-(d-1)}{w-(d-1)}}{\binom{p}{w}}.
\]
Using the combinatorial identity
\[
\frac{\binom{p-(d-1)}{w-(d-1)}}{\binom{p}{w}}
=
\frac{\binom{w}{d-1}}{\binom{p}{d-1}},
\]
we obtain the stated equivalent form.
\end{proof}

\begin{restatable}[Single-row masking given observed $(a,v)$]{lemma}{BayesMasking}\label{lem:bayes_masking}
Assume the instance distribution samples $S_1,\dots,S_g,S_{g+1}$ i.i.d.\ uniformly from $\{S\subseteq[p]:|S|=d-1\}$,
independently of the oracle samples. Fix a step $g$ and condition on a realized example $(a,v)$ with $\|v\|_0=w^\star$.
Let
\[
\rho:=\rho(w^\star)=\Pr_{S}\!\big[M_S(v)=1\big]=\frac{\binom{w^\star}{d-1}}{\binom{p}{d-1}},
\quad
m(a):=\sum_{j=1}^g a_j,
\quad
B(a,v):=\bigoplus_{j=1}^{g} a_j M_j(v).
\]
Then, marginalizing over the unknown prefix supports $(S_1,\dots,S_g)$, for each $r \in \{0,1\}$,
\[
\Pr\big[B(a,v)=r \mid a,v\big]
\;=\;
\tfrac{1}{2}\big[1+(-1)^r(1-2\rho)^{m(a)}\big].
\]
Moreover, $B(a,v)$ is independent of $b=M_{g+1}(v)$ given $(a,v)$, and since $y=B(a,v)\oplus b$ we have
\[
\left|\Pr[y=b\mid a,v]-\tfrac12\right|
\;=\;
\tfrac{1}{2}|1-2\rho|^{m(a)}.
\]
\end{restatable}

\begin{proof}
Condition on the realized $(a,v)$. For each active prefix address bit $a_j=1$, the random variable $M_j(v)$ is Bernoulli with mean $\rho$, because $S_j$ is uniform over supports of size $d-1$. These variables are independent across $j$ under the i.i.d.\ instance distribution. Terms with $a_j=0$ do not affect the XOR, so $B(a,v)$ is the parity of $m(a)$ independent Bernoulli$(\rho)$ variables.

Writing $m=m(a)$, for independent Bernoulli variables $Z_1,\dots,Z_m$ with $\Pr[Z_i=1]=\rho$, the parity $Z_1\oplus\cdots\oplus Z_m$ satisfies
\[
\mathbb E[(-1)^{Z_1\oplus\cdots\oplus Z_m}]
=\prod_{i=1}^m \mathbb E[(-1)^{Z_i}]
=(1-2\rho)^m.
\]
Since this expectation is $\Pr[\mathrm{parity}=0]-\Pr[\mathrm{parity}=1]$, we obtain
\[
\Pr[B(a,v)=r\mid a,v]
=\tfrac12\big[1+(-1)^r(1-2\rho)^{m(a)}\big].
\]
The next support $S_{g+1}$ is sampled independently of the prefix supports, so $b=M_{g+1}(v)$ is independent of $B(a,v)$ given $(a,v)$. With $y=B(a,v)\oplus b$, the event $y=b$ is exactly $B(a,v)=0$, hence
\[
\left|\Pr[y=b\mid a,v]-\tfrac12\right|
=\left|\Pr[B(a,v)=0\mid a,v]-\tfrac12\right|
=\tfrac12 |1-2\rho|^{m(a)}.
\]
\end{proof}

\begin{corollary}[Large-depth row-wise masking for fixed batches]\label{cor:batch_masking}
Fix a step-specific batch size $K$ and let
$m_k=\sum_{j=1}^g a_j^{(k)}$ be the number of active prefix bits in row $k$.
Then
\[
\Pr\!\left[\min_{k\le K} m_k < \frac{g}{4}\right]
\le
K\exp\!\left(-\frac{g}{16}\right).
\]
On the complementary event, Lem.~\ref{lem:bayes_masking} bounds the masking bias in every row by
\[
\tfrac{1}{2}|1-2\rho|^{g/4}.
\]
Thus, for fixed $K$ and $\rho$ bounded away from $0$ and $1$, both the probability of encountering a weakly masked row and the masking bias of every remaining row decay exponentially with depth.
\end{corollary}

\begin{proof}
For each row, $m_k\sim\mathrm{Binomial}(g,1/2)$. A standard Chernoff bound gives
\[
\Pr\!\left[m_k<\frac{g}{4}\right]
\le
\exp\!\left(-\frac{g}{16}\right).
\]
A union bound over the $K$ rows gives the first claim. On the complementary event, every row satisfies $m_k\ge g/4$. Lem.~\ref{lem:bayes_masking} therefore gives
\[
\left|
\Pr[y^{(k)}=b^{(k)}\mid a^{(k)},v^{(k)}]-\tfrac12
\right|
\le
\tfrac12|1-2\rho|^{m_k}
\le
\tfrac12|1-2\rho|^{g/4}
\]
for every row $k$. If $\rho$ is bounded away from $0$ and $1$, then $|1-2\rho|<1$ uniformly, so this bound decays exponentially in $g$.
\end{proof}

\section{The Monomial can be Recovered in Polynomial Time}~\label{appendix:poly-time-recovery}

Given $P_g$, a solver can form residual labels
\begin{equation}\label{eq:residual}
r^{(k)}
\;:=\;
y^{(k)} \oplus \bigoplus_{j=1}^{g} a^{(k)}_j M_j\!\big(v^{(k)}\big)
\;=\;
M_{g+1}\!\big(v^{(k)}\big),
\end{equation}
so recovery reduces to identifying the unknown degree-$(d-1)$ support $S_{g+1}$ from labeled payloads.

Let $K_+ := \{k: r^{(k)}=1\}$ and $T:=|K_+|$. Consider the decoder
\begin{equation}\label{eq:intersection_decoder}
\widehat S \;:=\; \bigcap_{k\in K_+} \mathrm{supp}\!\big(v^{(k)}\big),
\end{equation}
which outputs $\widehat S$ if $T\ge 1$ and $|\widehat S|=d-1$, and otherwise declares failure.

\begin{theorem}[Poly-time recovery under fixed-weight payloads]\label{thm:recoverability_poly}
Fix an instance with $p>d-1$ and assume payloads are i.i.d.\ uniform on
$\{v\in\{0,1\}^p:\|v\|_0=w^\star\}$ with $d-1\le w^\star<p$.
Let $\rho=\rho(w^\star)=\Pr[M_{g+1}(v)=1]$,
$\alpha:=\frac{w^\star-(d-1)}{p-(d-1)}\in[0,1)$,
and $T=|K_+|$. Then $\Pr[T=0]=(1-\rho)^K$. Moreover, for every integer $t\ge1$,
\begin{equation}\label{eq:recover_bound_T}
\Pr[\widehat S=S_{g+1}\mid T=t]
\ge
1-\min\{1,(p-(d-1))\alpha^t\}.
\end{equation}
\end{theorem}

\begin{proof}
If $r^{(k)}=1$, then $M_{g+1}(v^{(k)})=1$, which is equivalent to
$S_{g+1}\subseteq\mathrm{supp}(v^{(k)})$. Hence
$S_{g+1}\subseteq\widehat S$, and the decoder can fail only if $\widehat S$
contains an \emph{extraneous} coordinate $i\notin S_{g+1}$.

Fix such an $i$ and condition on $r=1$. Under the fixed-weight distribution,
the $d-1$ coordinates in $S_{g+1}$ are active, while the remaining
$w^\star-(d-1)$ active coordinates are chosen uniformly from the
$p-(d-1)$ coordinates outside $S_{g+1}$. Therefore
$\Pr[i\in\mathrm{supp}(v)\mid r=1]=\alpha$.

Now condition on $T=t$ and on the index set $K_+$. Because the original examples
are i.i.d., the payloads $\{v^{(k)}\}_{k\in K_+}$ are independent draws from
$(v\mid r=1)$. Thus the probability that a fixed extraneous coordinate appears
in every positive support is $\alpha^t$. A union bound over the $p-(d-1)$
possible extraneous coordinates gives
\[
\Pr[\widehat S\neq S_{g+1}\mid T=t,K_+]
\le
(p-(d-1))\alpha^t.
\]
Since this bound does not depend on the particular index set $K_+$, averaging
over $K_+$ conditional on $T=t$ gives
$\Pr[\widehat S\neq S_{g+1}\mid T=t]
\le\min\{1,(p-(d-1))\alpha^t\}$, which proves
\eqref{eq:recover_bound_T}.

Finally, $T=\sum_{k=1}^K\mathbf 1\{r^{(k)}=1\}$ and
$\Pr[r^{(k)}=1]=\rho$. Hence $T\sim\mathrm{Binomial}(K,\rho)$, so
$\Pr[T=0]=(1-\rho)^K$.
\end{proof}

\begin{restatable}[High-probability recovery with $K$ samples]{corollary}{RecoverabilityHP}\label{cor:recoverability_hp}
In the setting of Thm.~\ref{thm:recoverability_poly}, let
$\rho:=\rho(w^\star)=\Pr[M_{g+1}(v)=1]$ and
$\alpha:=\frac{w^\star-(d-1)}{p-(d-1)}\in[0,1)$.
Fix $\delta\in(0,1)$. For $\alpha\in(0,1)$, define
\[
T_0
:=
\left\lceil
\frac{\log\!\big(2(p-(d-1))/\delta\big)}
{\log(1/\alpha)}
\right\rceil,
\]
and if $\alpha=0$, take $T_0:=1$.
If
$K\ge \frac{1}{\rho}\max\{2T_0,\,8\log(2/\delta)\}$,
then the decoder in \eqref{eq:intersection_decoder} outputs
$\widehat S=S_{g+1}$ with probability at least $1-\delta$.
\end{restatable}

\begin{proof}
Let $T=|K_+|=\sum_{k=1}^K\mathbf 1\{r^{(k)}=1\}$.
Since $r^{(k)}=M_{g+1}(v^{(k)})$ and the payloads are i.i.d.,
$T\sim\mathrm{Binomial}(K,\rho)$.

By Thm.~\ref{thm:recoverability_poly}, for every $t\ge1$,
$\Pr[\widehat S\neq S_{g+1}\mid T=t]
\le (p-(d-1))\alpha^t$.
Therefore,
\[
\Pr[\widehat S\neq S_{g+1}]
\le
\Pr[T<T_0]
+
\Pr[\widehat S\neq S_{g+1}\mid T\ge T_0].
\]

For the first term, the assumed lower bound on $K$ gives
$K\rho/2\ge T_0$. A multiplicative Chernoff bound therefore yields
\[
\Pr[T<T_0]
\le
\Pr[T\le K\rho/2]
\le
\exp(-K\rho/8)
\le
\delta/2,
\]
where the final inequality uses
$K\rho\ge8\log(2/\delta)$.

For the second term, since $\alpha<1$, the bound
$(p-(d-1))\alpha^t$ is nonincreasing in $t$. Averaging over
$t\ge T_0$ therefore gives
\[
\Pr[\widehat S\neq S_{g+1}\mid T\ge T_0]
\le
(p-(d-1))\alpha^{T_0}
\le
\delta/2,
\]
by the definition of $T_0$. If $\alpha=0$, the same claim holds trivially for every $T\ge1$.
Combining the two bounds gives
$\Pr[\widehat S\neq S_{g+1}]\le\delta$.
\end{proof}

Lem.~\ref{lem:bayes_masking} shows a \emph{per-sample} obfuscation property: marginalizing over the unknown prefix supports,
each labeled example provides only exponentially small Bayes advantage about the next-term signal unless one conditions on the revealed
prefix. In the ideal balanced case $\rho(w^\star)=1/2$, this advantage is $0$ whenever $m(a)\ge 1$.
While the trivial leakage case $m(a)=0$ is theoretically possible, its probability $2^{-g}$ becomes exponentially rare as depth increases.
Lem.~\ref{lem:history_only} rules out history-only shortcuts under the instance distribution, since $S_{g+1}$ remains uniform given the
revealed prefix.

Finally, Thm.~\ref{thm:recoverability_poly} (and Corollary~\ref{cor:recoverability_hp}) show that a diligent solver can subtract the prefix mask to obtain residual labels $r^{(k)}=M_{g+1}(v^{(k)})$ and recover $t_{g+1}$ in polynomial time from
$K=O\!\big(\tfrac{1}{\rho(w^\star)}\log(p/\delta)\big)$ samples with failure probability at most $\delta$ (for fixed $d$ and constant $\rho(w^\star)$). Together, these properties justify using $\gamma_g$ in \eqref{eq:gamma_exact} as an operational measure of step success in our benchmark.

{\bf Efficient validation.\enspace} Given the instance specification (in particular $S_{g+1}$) and a candidate monomial $\tilde t$,
the validator parses $\tilde t$, rejects unless it contains exactly one address variable (which must be $a_{g+1}$) and exactly $d-1$ distinct payload variables, then accepts iff the parsed payload index set $\widetilde S$ equals $S_{g+1}$. This runs in time $O(|\tilde t| + d\log d)$ worst-case (or $O(|\tilde t| + d)$ expected with hashing).
\end{document}